\documentclass{article}
\usepackage[utf8]{inputenc}
\usepackage{amsmath}
\usepackage{amssymb}
\usepackage{amsthm}
\usepackage{centernot}
\usepackage[disable]{todonotes}
\usepackage[numbers]{natbib}
\usepackage[pdftex,colorlinks,pagebackref=true]{hyperref}
\newcommand{\cd}{{\cal D}}

\newcommand{\cc}{{\cal C}}

\newcommand{\ch}{{\cal H}}

\newcommand{\cx}{{\cal X}}
\newcommand{\cy}{{\cal Y}}

\newcommand{\cn}{{\cal N}}

\newcommand{\bw}{\mathbf{w}}

\newcommand{\bu}{\mathbf{u}}
\newcommand{\be}{\mathbf{e}}

\newcommand\ceil[1]{\lceil#1\rceil}
\newcommand{\norm}[1]{\left\lVert#1\right\rVert}
\newcommand{\parentheses}[1]{\left(#1\right)}
\newcommand{\HH}{\mathcal{H}}
\newcommand{\NN}{\mathcal{N}}

\newcommand{\CC}{\mathcal{C}}
\newcommand{\E}{\mathbb{E}}
\newcommand{\DD}{\mathcal{D}}
\newcommand{\XX}{\mathcal{X}}
\newcommand{\reals}{\mathbb{R}}
\newcommand{\fcnclass}{(\mathbb{R}^d)^\XX}
\newcommand{\cover}[4]{\mathcal{N}_#1(#2, #3, #4)}

\newcommand{\inner}[1]{{\left\langle #1 \right\rangle}}
\newcommand{\expectation}[1][ ]{\mathbb{E}_{#1}}
\newcommand{\sphere}{\mathbb{S}}
\newcommand{\floor}[1]{\left\lfloor #1 \right\rfloor}

\newcommand{\med}{\mathrm{med}}
\newcommand{\rep}{\mathrm{rep}}
\newcommand{\len}{\mathrm{len}}

\newtheorem{theorem}{Theorem}[section]
\newtheorem{lemma}[theorem]{Lemma}

\newtheorem{counter-example}[theorem]{Counter example}

\newtheorem{open question}[theorem]{Open question}

\newtheorem{definition}[theorem]{Definition}

\title{Approximate Description Length, Covering Numbers, and VC Dimension}

\author{Amit Daniely \and Gal Katzhendler}

\begin{document}

\maketitle

\begin{abstract}
    Recently, \citet{ADL2019} introduced a new notion of complexity called Approximate Description Length (ADL). They used it to derive novel generalization bounds for neural networks, that despite substantial work, were out of reach for more classical techniques such as discretization, Covering Numbers and Rademacher Complexity.
    
    In this paper we explore how ADL relates to classical notions of function complexity such as Covering Numbers and VC Dimension. We find that for functions whose range is the reals, ADL is essentially equivalent to these classical complexity measures. However, this equivalence breaks for functions with high dimensional range.
\end{abstract}

\section{Introduction}

Neural Networks are a widely used tool nowadays, despite the lack of theoretical background supporting their abilities to generalize well. Classical notions of learning guarantee generalization only if there are more examples that parameters. It is clear that a stronger assumption is needed to achieve tighter bounds, and indeed, different types of assumptions were used in order to fill this empirical-theoretical gap, including assumptions on robustness to noise \cite{noiserobust}, bias of the learning algorithm \cite{gunasekar2019implicit, soudry2018implicit}, and norm bounds on the weight's matrices \cite{neyshabur2015norm, neyshabur2018role}

The idea of Approximate Description Length \cite{ADL2019} was conceived as a part of the line of research working under assumptions that bound the magnitude of the network's weight matrices.
Consider for instance the class
\[\cn = \left\{W_t\circ\rho\circ W_{t-1}\circ\rho\ldots\circ \rho\circ W_{1} : W_1,\ldots,W_{t-1}\in M_{d\times d}, W_t\in M_{1,d}  \right\}
\]
Where the spectral norm of each $W_i$ is bounded by $O(1)$, the Frobenius norm is bounded by $R$, and $\rho$ is the sigmoid function $\frac{e^x}{1 + e^x}$ or the smoothened ReLU function $ \ln\left(1 + e^x\right)$.
While the line of work leading up to ADL managed to show a sample complexity of $\tilde O\left(\frac{d^2R^2}{\epsilon^2}\right)$ for this class, \cite{ADL2019} managed to show a sample complexity of $\tilde \Theta\left(\frac{dR^2}{\epsilon^2}\right)$, which is sublinear in the number of parameters of the network, for the first time.

In this paper we aim to understand ADL as a general approach to sample complexity. Does it fit as a general approach for bounding sample complexity? Can we tightly bound the sample complexity of most of the well known classes of functions? 

We show that for classes of functions to one dimension, i.e. whose image is in $\reals$, ADL is almost as good as covering numbers. Covering numbers are one the most general techniques to bound sample complexity: They have been used to yield sample complexity bounds for learning many classes of functions \cite{svmcover, vapnik1999nature}, and make up a big chunk of the techniques that have been tried to analyze the sample complexity of neural networks \cite{bartlett, bartlettpaper}. We show a semi-equivalence between Covering Numbers and ADL in one dimension. 

In higher dimensions matters become more complex, and in general this semi-equivalence does not hold. Indeed, we present a class of functions to high dimension, for which the sample complexity bound given by Covering Numbers is arbitrarily tighter than that obtainable by using the ADL technique.

We further show that this non-equivalence holds in an intrinsic way and more generally; for any norm $L$ that admits such a relationship between Covering Numbers and ADL, i.e. $log\left(\cover{L}{\HH}{m}{\varepsilon}\right)=O\left(\frac{\text{ADL}(\HH)}{\varepsilon^{2}}\right)$, the equivalence does not hold.

To conclude, our work shows that while ADL is a useful tool, it cannot be used to yield optimal sample complexity bounds to any class of functions. Yet all the examples for such gaps that we have seen are artificial and carefully constructed, so it is interesting to determine whether the ADL technique is applicable to other classes of interest.

\subsection{Notation}
We will use $\lesssim$ to denote an inequality that is true up to a constant. That is if we write $f(x)\lesssim g(x)$ we mean that there is some constant $C$ such that for any $x$, $f(x)\leq C g(x)$

\subsection{Covering Numbers, VC dimension and Generalization}

Fix a sample space $\cx$, a label space $\cy$, and a loss $\ell:\reals^d\times \cy \to [0,\infty)$. We say $\ell$ has some property (e.g. lipschitzness) if for any $y\in \cy$, $\ell(\cdot,y)$ has it. Fix some hypothesis class $\HH\subset\cy^\XX$, and let $\DD$ be some distribution $\XX\times \cy$. We define the distributional loss of some $h\in\HH$ over $\DD$ as $\ell_\cd(h)=\E_{(x,y)\sim\cd}\ell(h(x),y)$.
For a sample $S\in(X\times Y)^m$ we define the empirical loss of $h$ over $S$ as  $\ell_S(h)=\frac{1}{m}\sum_{i=1}^m\ell(h(x_i),y_i)$.
We define the {\em representativeness} of $S$ over $\HH$ as \[
\rep_\cd(S,\ch) = \sup_{h\in\ch}\ell_{\cd}(h) - \ell_S(h)\]
Note that if $\E_{S}\rep_\cd(S,\ch)\le \epsilon$ then any algorithm $A$ that minimizes the empirical loss (ERM) will satisfy $\expectation \ell_\DD(A(S))-\inf_{h\in\HH}\ell_\DD(h)\leq\varepsilon$



We next define Covering Numbers which are a central tool in the analysis of the complexity of classes of functions. This is true since the days of Kolmogorov's Metric Entropy \cite{kolmogorov}. Later, it became a prominent way of analyzing the sample complexity of learning models \cite{svmcover, vapnik1999nature,zhang2002covering}
In particular, Covering Numbers are a central technique that has been used in the analysis of the sample complexity of neural networks \cite{bartlett, bartlettpaper}.

\begin{definition} [Covering number]
    Fix a class $\HH$ of functions $\XX \mapsto \mathbb{R}^d$, an integer $m$, and $\varepsilon > 0$. 
    We define $\NN (\HH,m, \varepsilon)$ as the minimal integer for which the following holds. For every $A \subset \XX$ of size $\leq m$ there exists $\tilde{\HH} \subset \fcnclass$ with $|\tilde{\HH}|\leq \NN (\HH,m, \varepsilon)$ such that for any $ h \in \HH $ there is $\tilde{h} \in \tilde{\HH}$ with $\mathbb{E}_{x\in A}\norm{h(x)-\tilde{h}(x)}^2_\infty \leq \varepsilon^2$.
\end{definition}
We will also consider covering numbers defined relative to a general norm $\|\cdot\|_L$ rather than $\|\cdot\|_\infty$. In this case, the covering number will be denoted by $\cover{L}{\HH}{m}{\varepsilon}$. Finally, $\NN(\HH,\varepsilon)$ (resp. $\NN_L(\HH,\varepsilon)$) will stand for $\NN(\HH,|\XX|,\varepsilon)$ (resp. $\NN_L(\HH,|\XX|,\varepsilon))$.

The following lemma connects Covering Numbers to the sample complexity of learning:
\begin{lemma}[e.g. \cite{ADL2019}]\label{cover_to_gen}
Let $\ell:\reals^d\times\cy\to \reals$ be $L$-Lipschitz w.r.t. $\|\cdot\|_\infty$ and $B$-bounded. Assume that for any $0<\epsilon\le 1$, $\log\left(\NN(\ch,m,\epsilon)\right) \le \frac{n}{\epsilon^2} $
Then
\[
\E_{S\sim\cd^m}\rep_\cd(S,\ch) \lesssim   \frac{(L+B)\sqrt{n} }{\sqrt{m}} \log(m)
\]
Furthermore, with probability at least $1-\delta$,
\[
\rep_\cd(S,\ch) \lesssim   \frac{(L+B)\sqrt{n} }{\sqrt{m}} \log(m) + B\sqrt{\frac{2\ln\left(2/\delta\right)}{m}}
\]

\end{lemma}

We now define VC-dimension:

\begin{definition}
We say that a class $\HH\subset \{\pm 1\}^\XX$ shatters a set $A\subset \XX$ if $\HH|_A=\{h|_A|h\in \HH\} = \{\pm 1\}^A$. 
The VC-dimension of $\HH$, denoted $VC(\HH)$ is the maximal cardinality of a shattered set
\end{definition}
The following lemma connects VC-dimension to Covering Numbers:
\begin{lemma} \cite{Haussler} \label{lem:VC2Cover}
Let  $\HH\subset \{\pm 1\}^\XX$ with $VC$-dimension $d$, then \[\log\left(N(\ch,m,\epsilon)\right)=O(d/\varepsilon)\]
\end{lemma}
\subsection{Approximate Description Length}
We next  outline the definition and basic properties of Approximate Description Length.
Informally, $ADL(\HH)$ is the number of bits required to stochastically and approximately describe the functions in $\HH$. To define ADL we will need (i) the notion of $\sigma$-estimator which measures the extent to which a random function $\hat h$ approximates $h$ and (ii) the notion of a compressor to formally define what it means to describe 
``an approximation $\hat h$ of $h$ using $n$ bits"
We will consider a relaxed variation of the definition of ADL from \citet{ADL2019}. While relaxed, we will show that it still gives generalization guarantees similar to \citet{ADL2019}

\begin{definition} [$\sigma$-estimator]
Fix some probability measure $(\Omega,\mu)$.
A random function $\tilde{f}:\Omega \times \XX \to \reals^n$ is called a $\sigma$-estimator for $f:\XX \to \reals^n$ under a distribution $\XX \sim \DD$ if
\begin{enumerate}
    \item $\forall x\in \XX, \expectation \tilde{f}(x)=f(x)$
    \item $\forall u\in \sphere^{n-1}, \expectation[x\sim \DD]\expectation \inner{u,\tilde{f}(x)-f(x)}^2\leq \sigma^2$
\end{enumerate}
\end{definition}

Note that we require 2 to hold in expectation, and not for any $x$ as in \citet{ADL2019}. We thus relax their definition of $\sigma$-estimator.

We next define the notion of a compressor. Roughly speaking, a compressor for an hypothesis class $\HH$  takes as input a function $h$ from $\HH$, and stochastically outputs a compressed version $\CC(h)$ of $h$.
The compressed version has a short description via a string of bits. Furthermore, $\CC(h)$ is a $\sigma$-estimator of $h$.


\begin{definition}[Compressor]
    Fix $\DD$ a distribution over $\XX$.
    A $(\sigma, n)$-compressor for $(\HH, \DD)$ is a 4-tuple $\CC = (E, D, \Omega, \mu)$ where $\mu$ is a probability measure on $\Omega$, and $E, D$ are functions $E:\Omega\times\HH \to \{0,1\}^*$, $D:\{0,1\}^* \to \fcnclass$ s.t. $\CC=D\circ E$ and for any $h\in \HH$
    \begin{enumerate}
        \item $\CC(h)$ is a $\sigma$-estimator for $h$.
        \item $\mathbb{E}_{\omega \sim \mu} \text{len}(E(\omega, h)) \leq n$
    \end{enumerate}
\end{definition}
We are now ready to define approximate description length. 

\begin{definition} [ADL]
    We say that a class of functions $\HH \subset \fcnclass$ has approximate description length $n(m)$ if for any $A \subset \XX$ of size $\leq m$ and for any distribution $\DD$ on $A$, there exists a $(1, n(m))$-compressor for $\HH |_A$
\end{definition}
Additionally, we would like to introduce a simple but useful lemma
\begin{lemma}\label{var_lemma}\cite{ADL2019}
    Let $\ch$ be a set that has a (k, n)-compressor, then for any $1\ge \epsilon>0$ there exists an $\left(k \cdot \epsilon, n\lceil \epsilon^{-2} \rceil\right)$-compressor for $\ch$. 
\end{lemma}

Now, similarly to \cite{ADL2019} we connect ADL to learnability through covering numbers. First, in the single dimensional case
\begin{theorem}\label{learnability}
Fix a class $\ch$ of functions from $\cx$ to $\reals$  with approximate description length $n$. Then,
\[
\log\left(N(\ch,m,\epsilon)\right) \lesssim {n\left\lceil \epsilon^{-2} \right\rceil}
\]
Hence, if $\ell:\reals^d\times\cy\to \reals$ is $L$-Lipschitz and $B$-bounded, then for any distribution $\cd$ on $\cx\times\cy$
\[
\E_{S\sim\cd^m}\rep_\cd(S,\ch) \lesssim   \frac{(L+B)\sqrt{n} }{\sqrt{m}} \log(m)
\]
Furthermore, with probability at least $1-\delta$,
\[
\rep_\cd(S,\ch) \lesssim   \frac{(L+B)\sqrt{n} }{\sqrt{m}} \log(m) + B\sqrt{\frac{2\ln\left(2/\delta\right)}{m}}
\]
\end{theorem}

And second, in the multidimensional case
\begin{theorem}\label{learnability-multi}
Fix a class $\ch$ of functions from $\cx$ to $\reals^d$ with approximate description length $n$. Then,
\[
\log\left(N_\infty(\ch,m,\epsilon)\right) \le {n\left\lceil 16\epsilon^{-2} \right\rceil}\lceil\log(dm)\rceil
\]
Hence, if $\ell:\reals^d\times\cy\to \reals$ is $L$-Lipschitz w.r.t. $\|\cdot\|_\infty$ and $B$-bounded, then for any distribution $\cd$ on $\cx\times\cy$
\[
\E_{S\sim\cd^m}\rep_\cd(S,\ch) \lesssim   \frac{(L+B)\sqrt{n\log(dm)} }{\sqrt{m}} \log(m)
\]
Furthermore, with probability at least $1-\delta$,
\[
\rep_\cd(S,\ch) \lesssim   \frac{(L+B)\sqrt{n\log(dm)} }{\sqrt{m}} \log(m) + B\sqrt{\frac{2\ln\left(2/\delta\right)}{m}}
\]
\end{theorem}

As described in the following theorem, ADL and techniques for calculating ADL were used to give novel generalization bounds for neural networks.

\begin{theorem}\cite{ADL2019}
Fix constants $r>0$, $t>0$ and a strongly bounded activation $\sigma$.
Then, for every choice of matrices $W_i^0 \in M_{d_{i},d_{i-1}},\;i=1,\ldots,t$ with $d: = \max_{i}d_i$ and $\max_{i}\|W^0_i\| \le  r$ we have that the approximate description length of $\ch = \cn^\rho_{r,R}(W^0_1,\ldots,W^0_t)$ is
\[
dR^2 O\left(\log^{t}(d)\right)\log(md) =  \tilde{O}\left(dR^2\right)
\]
In particular, if $\ell:\reals^{d_t}\times\cy\to \reals$ is bounded and Lipschitz w.r.t. $\|\cdot\|_\infty$, then for any distribution $\cd$ on $\cx\times\cy$
\[
\E_{S\sim\cd^m}\rep_\cd(S,\ch) \le  \tilde O\left( \sqrt{\frac{dR^2  }{m}}  \right)
\]
Furthermore, with probability at least $1-\delta$,
\[
\rep_\cd(S,\ch) \lesssim    \tilde O\left( \sqrt{\frac{dR^2  }{m}}  \right) + O\left(\sqrt{\frac{\ln\left(1/\delta\right)}{m}}\right)
\]
\end{theorem}
We note that the result was proved via \citet{ADL2019} definition of ADL, but remains valid if our definition is used.

\subsection{Our Contribution: ADL vs Covering numbers}

\todo{Try to make this section more attractive.}

    
Given that ADL has been successfully used to break barriers in the analysis of the sample complexity of Neural Networks, and is currently the only known technique that can derive several state-of-the-art results,
it is natural to ask where does ADL stand in the landscape of statistical complexity. Is it a stronger notion than some? A weaker one? Does it always give the correct asymptotic of sample complexity?

Our paper revolves around these questions. We start with one of the simplest and perhaps the most studied case of binary classification. Here, the VC-dimension characterize the sample complexity. So, we can ask how do VC-dimension and ADL relate. Our first result shows that ADL and VC-dimension are essentially equivalent:
\begin{theorem}\label{VC2ADL}
For a universal constant $C>0$, we have that $\frac{1}{C} ADL(\HH) \le VC(\HH)  \le C\cdot ADL(\HH)$ for any $\HH \subset \{0,1\}^\XX$
\end{theorem}
The right inequality follows from the fact that VC-dimension tightly characterizes the sample complexity of learning a binary class.
Note that naïvely, using the Sauer-Shelah lemma we can get an ADL of $n(m)=O(d\cdot \log (m))$ for a class of VC-dimension $d$; this is by encoding each hypothesis deterministically and separately. Our contribution here is therefore to show that the dependence on $m$ can be removed

The above result could give rise to the meta-conjecture that ADL characterizes the sample complexity in many regimes of interest. 
The following result corroborates this, as it establishes a semi-equivalence between ADL and Covering Numbers:
\begin{theorem}\label{Cover2ADL}
Fix some small $\alpha>0$ and let $\HH \subset [0,1]^\XX$ be a function class with $log\left(\cover{2}{\HH}{m}{\varepsilon}\right)=O(\frac{d}{\varepsilon^{2-\alpha}})$. Then\footnote{The hidden constant in the big-O notation depends only on $\alpha$.} $ADL(\HH)=O(d)$
\end{theorem}
Together with Theorem \ref{learnability} (i.e. ADL of $d \implies$ log-cover of $\frac{d}{\varepsilon^2}$), this almost gives us an equivalence between log-Covering Numbers of order $\frac{n}{\varepsilon^2}$ and ADL of order $n$. This further strengthens the aforementioned meta-conjecture.

Note that the results so far are only for classes of functions to the reals. We set out to verify whether this semi-equivalence holds in the case of functions to a high dimensional range. Here, we find that it does not hold. As we show next, Covering Numbers cannot be used to upper bound the ADL. This implies in particular that sample complexity bounds obtained via ADL are not always tight, and tighter bounds might be obtained via covering numbers.
\begin{theorem}\label{CoverNotADL}
For any $0<\alpha<1$ there exists $\HH \subset \left([0,1]^n\right)^\XX$ such that $log\left(\cover{2}{\HH}{m}{\varepsilon}\right)=O(1)$. But $ADL(\HH)=\Omega(d/\log(d))$
\end{theorem}

Given the above result, one might hypothesize that the Euclidean norm is not the ``correct" norm in the context of ADL. That is, if we will measure ADL via a different norm, we might get an equivalence between covering and ADL similarly to theorem \ref{Cover2ADL}. That is, there might be a norm $\|\cdot\|_L$ such that 
$log\left(\cover{L}{\HH}{m}{\varepsilon}\right)=O(\frac{d}{\varepsilon^{2-\alpha}}) \Rightarrow ADL(\HH)=O(d)$ and $ADL(\HH)=O(d) \Rightarrow \log \left(\cover{L}{\HH}{m}{\varepsilon}\right)=O(\frac{d}{\varepsilon^{2}})$. Having such a norm, might be a useful tool for understanding ADL, and therefore sample complexity in general. Unfortunately, the following theorem shows that it is not the case, and that in high dimensions, ADL is a fundamentally different concept. \todo{say, maybe in the introduction, that intuitively ADL and Covering seems similar. Thus, it make sense that they are equivalent, and we just need to find the correct norm. We show that surprisingly this is not the case. This might shed some light on the fact that ADL succeeded to break barriers that were not broken via covering number techniques}


\begin{theorem}\label{CoverNotADLGeneralized}
Fix a  norm $\|\cdot\|_L$ such that for any class $\HH$ with $ADL(\HH)=d$ we have that $log\left(\cover{L}{\HH}{m}{\varepsilon}\right) \leq \frac{d}{\varepsilon^2}$. Then, for any $d$ there exists a set $\HH \subset [0,1]^n$ for $n= d^4$
such that $log\left(\NN_L(\HH,\epsilon)\right)\le \frac{1}{\epsilon^2}$. But $ADL(\HH)\ge d$
\end{theorem}

This is a bit discouraging, and due to phenomena of high dimensional geometry.
Overall it seems like ADL could be beneficial to use in some cases for the sample complexity analysis of learning models.
\section{Proofs of the Learnability Theorems}
\subsection{Proof of Theorem \ref{learnability}}
\begin{proof}
    Fix a set $A\subset \XX$ of size $\leq m$ and let $\DD$ be a the uniform distribution on $A$. Since $\HH$ has ADL of $n$, there exists a $(1, n$-compressor for $(\HH,\DD)$. 
    By \ref{var_lemma}, this implies the existence of an $\parentheses{\frac{\varepsilon}{2}, \ceil{4\varepsilon^{-2}} n}$-compressor for $(\HH, \DD)$, which will be denoted by $\CC$. 
    It is enough to show that for any $h\in \HH$ there is $\omega\in\Omega$ such that $\len(E(\omega,h)) \le 2\ceil{4\varepsilon^{-2}} n(m)$ and $\mathbb{E}_{x \sim \DD} (\CC_\omega(h)(x)-h(x))^2 \leq \varepsilon^2$.
    Indeed, by Markov's inequality, we have that
    \[
    P_{\omega\sim\mu}(\text{len}(E(\omega, h)) \geq 2\ceil{4\varepsilon^{-2}} n(m)) \leq \frac{1}{2}
    \]
    Using Markov's inequality again, together with the fact that
    \[
     \mathbb{E}_{\omega \sim \mu}\mathbb{E}_{x \sim \DD} (\CC_\omega(h)(x)-h(x))^2 \leq \frac{\varepsilon^2}{4}
    \]
    we get
    \[
    P_{\omega\sim\mu}(\mathbb{E}_{x \sim \DD} (\CC_\omega(h)(x)-h(x))^2 \geq \varepsilon^2) \leq \frac{1}{4}
    \]
    We conclude that for $\omega\sim\mu$, with probability at least $\frac{1}{4}$, $\len(E(\omega,h)) \le 2\ceil{4\varepsilon^{-2}} n(m)$ and $\mathbb{E}_{x \sim \DD} (\CC_\omega(h)(x)-h(x))^2 \leq \varepsilon^2$. In particular, there exists such an $\omega$.
\end{proof}
\subsection{Proof of Theorem \ref{learnability-multi}}
\begin{lemma}\cite{ADL2019}\label{lem:median}
Let $X_1,\ldots,X_n$ be independent r.v. with that that are $\sigma$-estimators to $\mu$. Then
\[
\Pr\left(|\med(X_1,\ldots,X_n)-\mu|>k\sigma\right) <  \left(\frac{2}{k}\right)^n
\]
\end{lemma}
We are now ready to prove the theorem
\begin{proof}
    Denote $k=\lceil \log (d) \rceil$. Fix a set $A\subset \XX$ of size $\leq m$ and let $\DD$ be a the uniform distribution on $A$. Since $\HH$ has ADL of $n$, there exists a $(1, n$-compressor for $(\HH,\DD)$. 
    By \ref{var_lemma}, this implies the existence of an $\parentheses{\frac{\varepsilon}{4}, \ceil{16\varepsilon^{-2}} n}$-compressor for $(\HH, \DD)$, which will be denoted by $\CC'$. Define
\[
(\cc_{\omega_1,\ldots,\omega_k}h)(x) = \med\left((\cc'_{\omega_1}h)(x),\ldots,(\cc'_{\omega_k}f)(x)  \right)
\]
    It is enough to show that for any $h\in \HH$ there is ($\omega_1,\ldots \omega_k) \in\Omega^k$ such that $\len(E(\omega,h)) \le 2k\ceil{16\varepsilon^{-2}} n(m)$ and $\mathbb{E}_{x \sim \DD} \norm{\CC_\omega(h)(x)-h(x)}_\infty^2 \leq \varepsilon^2$.
    Indeed, by Markov's inequality, we have that
    \[
    P_{\omega\sim\mu}(\text{len}(E'(\omega, h)) \geq 2\ceil{16\varepsilon^{-2}} n(m)) \leq \frac{1}{2}
    \]
    Therefore
    \[
    P_{\omega\sim\mu}(\text{len}(E(\omega, h)) \geq 2k\ceil{16\varepsilon^{-2}} n(m)) \leq 2^{-k}
    \]
    Using \ref{lem:median}, together with the fact that $\forall u\in \sphere^{d-1}$
    \[
     \mathbb{E}_{\omega \sim \mu}\mathbb{E}_{x \sim \DD} \inner{u ,\CC_\omega(h)(x)-h(x)}^2 \leq \frac{\varepsilon^2}{4}
    \]
    we get
    \[
    P_{\omega\sim\mu}(\mathbb{E}_{x \sim \DD} \norm{\CC_\omega(h)(x)-h(x)}_\infty^2 > \varepsilon^2) < d \cdot 2^{-k} \leq 1
    \]
    We conclude that for $(\omega_1,\ldots, \omega_k)\sim\mu^k$, with positive probability, $\len(E(\omega,h)) \le 2k\ceil{16\varepsilon^{-2}} n(m)$ and $\mathbb{E}_{x \sim \DD} \norm{\CC_\omega(h)(x)-h(x)}_\infty^2 \leq \varepsilon^2$. In particular, there exists such an $(\omega_1,\ldots, \omega_k)$.
\end{proof}

\section{One Dimensional Output}

\subsection{Proof of Theorem \ref{Cover2ADL}}

\begin{theorem}
Fix some small $1>a>0$ and let $\HH \subset [0,1]^\XX$ be a function class with $log\left(\cover{2}{\HH}{m}{\varepsilon}\right) \le \frac{d}{\varepsilon^{2-a}}$. Then\footnote{The hidden constant in the big-O notation depends only on $a$.} $ADL(\HH)=O(d)$
\end{theorem}

\begin{proof}

Fix some $m$ and $A \subset X$ of size $\leq$ m. 
Define $\varepsilon_n = 2^{-n}$ and choose for any $n$ an $\varepsilon_n$-cover $\HH_i$ of log-size $d/\varepsilon_n^{2-a}$ for $\HH$. 
Assume that $\HH_0 = \{g_0\}$ where $g_0\equiv 0$ (note that $\HH_0$ is indeed a $2^{0}$-cover of $\HH$).
For any $h\in\HH$ denote by $h_n\in \HH_n$ closet function to $h$ in $\HH_n$. Now, for $n\geq 1$ define
\[
\CC^n(h)=\frac{1}{\varepsilon_n^{2-a/2}}(h_n-h_{n-1})
\]
Note that $\CC_n$ can be encoded using $\frac{2d}{\varepsilon_n^{2-a}}$ bits.
Finally, define a compressor
\[
\CC(h)(x) = 
\begin{cases}
C_n(h)(x) &\text{ w.p. } \varepsilon_n^{2-a/2}\text{ for }n\geq 1 \\
0         &\text{ w.p. } 1-\frac{2^{-2+a/2}}{1-2^{-2+a/2}}

\end{cases}
\]
Note that
\begin{eqnarray*}
\expectation \CC(h) &= &\expectation \sum_{n=1}^\infty \varepsilon_n^{2-a/2} \CC_n(h)
\\
&=&\sum_{n=1}^\infty h_n-h_{n-1}
\\
&=&\lim_{n\to \infty}h_n 
\\
&=&  h
\end{eqnarray*}
Furthermore, the expected number of bits needed to encode $\CC$ is
\begin{align*}
    \expectation[\omega]\text{len}(E(\omega,h) & = \sum_{n=1}^\infty \varepsilon_n^{2-a/2} \cdot \frac{2d}{\varepsilon_n^{2-a}} \\ 
    &= \sum_{n=1}^\infty 2d\varepsilon_n^{a/2}\\
    &= \left(2\cdot \frac{ 2^{-a/2}}{1-2^{-a/2}}\right)\cdot d
\end{align*}
Finally, let us calculate $\CC$'s variance:
\begin{eqnarray*}
    \expectation[x]\expectation[\omega]\left(\CC(h)(x)-h(x)\right)^2 &\leq& \expectation[x]\expectation[\omega]\left(\CC(h)(x))\right)^2 \\
                                         &=&\expectation[x] \sum_{n=1}^\infty \varepsilon_n^{2-a/2} \left( \CC_n(h)(x)\right)^2\\ 
                                         &=& \expectation[x] \sum_{n=1}^\infty \varepsilon_n^{2-a/2}  \left(\frac{1}{\varepsilon_n^{2-a/2}}(h_n(x)-h_{n-1}(x))\right)^2\\ 
                                         &=&\expectation[x] \sum_{n=1}^\infty \frac{1}{\varepsilon_n^{2-a/2}} \left(h_n(x)-h_{n-1}(x)\right)^2\\
                                         &=& \sum_{n=1}^\infty \frac{1}{\varepsilon_n^{2-a/2}} \expectation[x]\left(h_n(x)-h_{n-1}(x)\right)^2\\
                                         &\overset{\Delta\text{-ineq}}{\leq}&\sum_{n=1}^\infty \frac{1}{\varepsilon_n^{2-a/2}} \left(2\expectation[x](h_n(x)-h(x))^2+2\expectation[x](h(x)-h_{n-1}(x))^2\right)\\
                                         & \leq& \sum_{n=1}^\infty \frac{1}{\varepsilon_n^{2-a/2}} \left(2\varepsilon_n^2+2\varepsilon_{n-1}^2\right)\\
                                         & \leq &\sum_{n=1}^\infty \frac{4\varepsilon_{n-1}^2}{\varepsilon_n^{2-a/2}}\\
\end{eqnarray*}
Since $\varepsilon_n = 1/2^n$ we get 
\begin{align*}
    \expectation[x]\expectation[\omega]\left(\CC(h)(x)-h(x)\right)^2 &\leq \sum_{n=1}^\infty \left( \frac{4\cdot 2^{-2n-2}}{2^{-2n+an/2}}\right)\\
    & = \sum_{n=1}^\infty \left( \frac{1}{2^{an/2}}\right) = O(1)
\end{align*}

Note that as the variance is constant this implies the existence of a $(1,O(d))$-compressor by lemma \ref{var_lemma}.
\end{proof}

\subsection{Proof of Theorem \ref{VC2ADL}}
With the following lemma, the proof of Theorem \ref{VC2ADL} becomes a corollary of Theorem \ref{Cover2ADL}.
\begin{lemma} [\cite{Haussler}] \label{packinglemma}
    For any $A \subset \{0,1\}^m$, $\mathcal{N}(A, \varepsilon)=O\left(\frac{VC(A)}{\varepsilon}\right)$
\end{lemma}

\section{Multi-Dimensional Output}
In the following subsections we will consider classes of functions from a single point to $\reals^n$.
It is therefore will be natural to associate these function classes with subsets of $\reals^n$, and to accordingly apply the notions of ADL and covering numbers to subsets of $\reals^n$.

\subsection{Proof of Theorem \ref{CoverNotADL}} \label{FirstInequivalence}

The idea behind the following proof is to take the discrete cube of dimension $d$, and isometrically embed it in a space of higher dimension $n$, such that the embedded cube will be contained in $[-\epsilon,\epsilon]^n$ for some small $\epsilon$

\begin{definition}
A \emph{hadamard} matrix is an orthogonal $m\times m $ matrix with entries in $\pm\frac{1}{\sqrt{m}}$.
A $2^n\times 2^n$ hadamard matrix $H_n$ can be constructed recursively by defining $H_0 = \begin{pmatrix} 1\end{pmatrix}$ and $H_n=\frac{1}{\sqrt{2}}\begin{pmatrix}H_{n-1} & H_{n-1} \\ H_{n-1} & -H_{n-1}\end{pmatrix}$ for $n\ge 1$.
\end{definition}

\begin{theorem}
For any $0<\alpha<1$ and $d$ there exists $\HH \subset [0,1]^n$
for $n = O\left(d^{2 + 2/\alpha}\right)$
such that $\log\left(\NN(\HH, \varepsilon)\right)\le \frac{1}{\epsilon^{\alpha}}$ but $ADL(\HH)\ge d$. 
\end{theorem}

\begin{proof}

Let $S=\{\pm 1\}^d$. Let $A\in \reals^{n \times d}$ be a matrix whose columns are some choice of different columns from a Hadamard matrix of size $n\times n$. The parameter $n$ will be a power of two and its value will be specified later.

Now, let $\HH$ be the image of $S$ under $A$, i.e 
$$\HH=\left\{A\cdot s | s\in S\right\}\subset \reals^n$$
We have that for $h\in \HH$, $\norm{h}_\infty \leq \frac{d}{\sqrt{n}}$.
Hence, for $ \varepsilon \geq \frac{d}{\sqrt{n}}$ we have $\log\left(\NN (\HH,\varepsilon)\right)=  0$ and in particular $\log\left(\NN (\HH,\varepsilon)\right) \le 1/\epsilon^\alpha$. As for $\varepsilon < \frac{d}{\sqrt{n}}$ we have the trivial bound $\log\left(\NN (\HH,\varepsilon)\right) \le d$. Now, choosing $n$ to be the smallest power of $2$ such that $n \ge  d^{2 + 2/\alpha}$, we will get that for $\varepsilon < \frac{d}{\sqrt{n}}$
\[
1/\epsilon^\alpha \ge \frac{n^{\alpha/2}}{d^\alpha} \ge d \ge \log\left(\NN (\HH,\varepsilon)\right)
\]
Thus, overall, $\log\left(\NN (\HH,\varepsilon)\right) \le 1/\epsilon^\alpha$. Finally, we  show next that ADL is invariant to orthogonal transformations, and hence, $ADL(\HH) = ADL(S) = \Omega(d/\log(d))$, where the last inequality follows from \ref{learnability-multi}, and the fact that for a constant $\varepsilon$ the log-covering of $S$ is of size $d$.
\begin{lemma}\label{lemma_orthonormal_invariance}
Fix $\HH\subset \left(\reals^n\right)^{\XX}$ and a matrix $U\in \reals^{k \times n}$ with orthonormal columns. Then $ADL(U\circ \HH) = ADL(\HH)$ 
\end{lemma}
\begin{proof}
Let $\CC'$ be the realizing $(1,ADL(\HH))$-compressor for $\HH$, for $h\in U\circ \HH$ define $\CC(h)=U\circ \CC'(U^T \circ h)$.
Now let $U \circ h \in U \circ \HH $, we have that 
\begin{align*}
    \expectation \CC(U\circ h)(x)&=\expectation U\circ \CC'(U^T \circ U \circ h)(x)=\\
 & = \expectation U\circ \CC'(h)(x) = \\ 
 & = Uh(x)
\end{align*}

Furthermore, $\forall u\in \sphere$ we have that 
\begin{align*}
    \expectation \inner{u, \CC(U\circ h)(x)-Uh(x)}^2& = \expectation \inner{u, U\circ \CC'(h)(x)-Uh(x)}^2= \\ 
 & = \expectation \inner{U^Tu,\CC'(h)(x)-h(x)}^2 \overset{*}{\leq} 1
\end{align*}
Where * is as $U^Tu$ is a unit vector.
\end{proof}
\end{proof}
\subsection{Proof of Theorem \ref{CoverNotADLGeneralized}}

\begin{theorem}
Fix a  norm $\|\cdot\|_L$ such that for any class $\HH$ with $ADL(\HH)=d$ we have that $log\left(\cover{L}{\HH}{m}{\varepsilon}\right) \leq \frac{d}{\varepsilon^2}$. Then, for any $d$ there exists a set $\HH \subset [0,1]^n$ for $n= d^4$
such that $log\left(\NN_L(\HH,\epsilon)\right)\le \frac{1}{\epsilon^2}$. But $ADL(\HH)\ge d$
\end{theorem}

First, let us introduce a few other useful notions from \cite{ADL2019}
\begin{definition}
Let $\bw\in\reals^d$ be a vector. A {\em random sketch of $\bw$} is a random vector $\hat \bw$ that is sampled as follows.  Choose $i$ w.p.
$p_{i} = \frac{w_{i}^2}{2\| \bw\|^2} + \frac{1}{2d}$. Then, w.p. $ \frac{w_{i}}{p_{i}} - \left\lfloor \frac{w_{i}}{p_{i}} \right\rfloor$ let $b=1$ and otherwise $b=0$.
Finally, let $\hat \bw = \left(\left\lfloor \frac{w_{i}}{p_{i}} \right\rfloor + b\right) \be_{i}$. A {\em random $k$-sketch of $\bw$} is an average of $k$-independent random sketches of $\bw$. A random sketch and a random $k$-sketch of a matrix is defined similarly, with the standard matrix basis instead of the standard vector basis.
\end{definition}
The following lemma shows that an sketch $\bw$ is a $\sqrt{ \frac{1}{4} + 2\|\bw\|^2}$-estimator of $\bw$.

\begin{lemma}\cite{ADL2019}\label{lem:sketch}
Let $\hat \bw$ be a random sketch of $\bw\in\reals^d$. Then,
\begin{enumerate}
\item
$\E\hat\bw = \bw$
\item
For any $\bu\in\sphere^{d-1}$, $\E\left(\inner{\bu,\hat \bw} -\inner{\bu, \bw} \right)^2 \le  \E\inner{\bu,\hat \bw}^2 \le \frac{1}{4} + 2\|\bw\|^2$
\end{enumerate}
\end{lemma}

We will use the following lemma:
\begin{lemma}\label{adl_of_ball}
The class $Q=\left\{ q \in \reals ^d |\; \|q\|_2 \leq M \right\}$ has an approximate description length $ \leq C\cdot M^2\log(dM)$ for some constant $C>0$
\end{lemma}
\begin{proof}
Consider the compressor that uses the random sketch of a vector for every $q\in Q$. For every $q\in Q$, denote by $u_{q_1}, \ldots, u_{q_d}$ the $d$ axis-aligned vectors that are used to sketch it by \ref{lem:sketch}. Let us upper bound the cardinality of $U=\bigcup_{q\in Q}\{u_{q_1}, \ldots,u_{q_d}\}$. Indeed, in every dimension we have that the outermost points have a norm of $\floor{\frac{4dq_i\|q\|^2}{2dq_i^2+2\|q\|^2}}+1\leq 5dM$ and so overall taking all axes and both positive and negative directions of every axis we have that $|U|\leq 2\cdot d \cdot 5 d M$ and so together with lemma \ref{var_lemma}, our compressor uses at most $C\cdot M^2\log(dM)$ bits for some $C>0$.
\end{proof}

We are now ready to prove the theorem. The general idea behind the proof is to use the assumed upper bound on the covering number by the ADL, i.e. \[\log\left(\cover{{L}}{\HH}{m}{\varepsilon}\right)=d/\varepsilon^{2}\] 
on a large \emph{euclidean} ball. This is to lower bound the volume of one the \emph{L-normed} balls that is a part of the cover. We then find a set of long orthogonal vectors inside the lower bounded \emph{L-normed} ball, by assuming they don't exist and getting a contradiction, and specifically that the \emph{L-normed} ball is contained within an euclidean ball of smaller volume than the previously calculated lower bound. We then use these vectors as a lower dimensional cube of high radius, and proceed similarly to the proof of theorem \ref{CoverNotADL}
\begin{proof}

Let $\norm{\cdot}_L$ be a norm\todo{sequnce of norms} such that for any set\todo{explain that this is stronger} $\HH\subset\reals^n$ with it holds that
\[
\log\left(\NN_L(\HH)\right)\le ADL(\HH)/\varepsilon^{2}
\]
Let $B_{\sqrt{k}/log(k),k}$ be the Euclidean ball of radius $\sqrt{k}/log(k)$ centered at $0$ in $\reals^k$.
By lemma \ref{adl_of_ball} we have that 
\[
ADL(B_{\sqrt{k}/log(k),k}) \leq \frac{C \cdot k\cdot \log(k\cdot k^{1/2}/\log(k))}{\log^2(k)}\leq C'\frac{k}{\log(k)}
\] 
for some universal constant $C'>0$, and so for an appropriate universa constant $\varepsilon$ we have that
\[
\log\left(
\NN_L(B_{\sqrt{k}/log(k),k},\epsilon)\right)\leq k
\]
Let $\{B_{\epsilon,k}^L(x_1),\ldots,B_{\epsilon,k}^L(x_{2^{k}})\}$ a cover of $B_{\sqrt{k}/log(k),k}$, where $B_{\epsilon,k}^L(x)$ is the $L$-normed ball of radius $\epsilon$ and $k$ dimensions, centered at $x$. Let
\[
i'=\arg \max_{i\in [2^k]} Vol(B_{\epsilon,k}^L(x_{i})\cap B_{\sqrt{k}/log(k),k})
\]
and denote $\tilde{B}=B_{\epsilon,k}^L(x_{i'})\cap B_{\sqrt{k}/log(k),k}$. Note that
\[
Vol(\tilde{B}) \geq 2^{-k} Vol (B_{\sqrt{k}/log(k),k})
\]
By the formula for the volume of the $n$-ball and Stirling's approximation we have that
\[
Vol(B_{r,d})\sim \frac{1}{\sqrt{\pi d}}\left(\frac{2\pi e}{d}\right)^{\frac{d}{2}}r^d
\]
Hence, 
$Vol(B_{\sqrt{k}/log(k),k})\sim\left(\frac{1}{\sqrt{k\pi}}\left(2\pi e\right)^{k/2}log^{-k}(k)\right)$. It follows that
\[
Vol(\tilde{B}) \geq 2^{-k}  Vol (B_{\sqrt{k}/log(k),k}) \geq \Omega \left(\frac{1}{\sqrt{k\pi}}\left(\frac{\pi}{log^2(k)}\right)^{k/2}\right)
\]
W.l.o.g and for the sake of notation simplicity, assume $x_{i'}=0$ i.e. $\tilde{B}=B_{\epsilon,k}^L(0)\cap B_{\sqrt{k}/log(k),k}$. This is only to simplify the expression for lengths of vectors inside that ball.


Let $u_1\in \tilde{B}$ be a vector with a maximal Euclidean norm. We have that $\|u_1\| \geq \frac{\left(\frac{k}{2\pi e}\right)^{1/2}}{\log(k)}$ as otherwise $\tilde{B} \subset B_{r,k}(0)$ with $Vol(B_{r,k}(0))\leq o \left(\frac{1}{\sqrt{k\pi}}\left(\frac{\pi}{log^2(k)}\right)^{k/2}\right)$ for $r<\left(\frac{k}{2\pi e}\right)^{1/2}log^{-1}(k)$.
Denote $\tilde{B}_1=\tilde{B}\cap (span\{u_1\}^\perp)$. Note that $Vol(\tilde{B}_1)\geq \Omega \left(\frac{1}{2\pi k}\left(\frac{\pi}{log^2(k)}\right)^{k/2}\right)$ as $\tilde{B_1}\subset B_{\sqrt{k},k}$, so the volume of the projection to the orthogonal $k-1$ dimensions can shrink by up to a factor of $2\sqrt{k}$. Now let $u_2\in \tilde{B}_1$ be a vector of maximal norm, then 
\[
\|u_2\|\geq\frac{1}{2\pi k}^{\frac{1}{k-1}}\left(\frac{k-1}{2\pi e}\right)^{1/2}log^{-k/(k-1)}(k)
\]
As otherwise $\tilde{B}_1$ will be contained in a ball of a smaller volume than its own; namely, we want 
\[
\frac{1}{2\pi k}\left(\frac{\pi}{log^2(k)}\right)^{k/2} \leq Vol(\tilde{B}_1) \leq \frac{1}{\sqrt{\pi (k-1)}}\left(\frac{2\pi e}{k-1}\right)^{\frac{k-1}{2}}\|u_2\|^{k-1}
\]
Overall by denoting $\tilde{B}_i=\tilde{B}\cap (span\{u_1,\ldots, u_i\}^\perp)$ and using the same argument we get
\[\|u_{i+1}\|\geq\left(8k\right)^{\frac{-i-1}{2k-2i}}\left(\frac{k-i}{2\pi e}\right)^{1/2}log^{-k/(k-i)}(k)\]
Repeating the process $k^{\frac{1}{4}}$ times and for large enough $k$, we get that $$\|u_1\|,\ldots,\|u_{k^{\frac{1}{4}}}\|\geq \left(\frac{k}{16\pi e}\right)^{1/2}log^{-2}(k)$$
Now notice that $conv(\pm u_1,\ldots,\pm u_{k^{\frac{1}{4}}}) \subset B_{\epsilon,k}^L(0)$ and therefore a $k^{\frac{1}{4}}$-dimensional cube of radius $C_0\cdot k^{\frac{1}{4}}\cdot log^{-2}(k)$ for some constant $C_0$, denote this cube by $L$. By choosing $d=k^{\frac{1}{4}}$, taking the set $S=\{\pm1\}^d$, and $A$ to be the orthogonal matrix that aligns $S$ with the corners of $L$, and denote $$\HH=\left\{A\cdot s | s\in S\right\}\subset \reals^k$$. We have that
$$log\left(\NN_L(\HH,\varepsilon)\right)=\begin{cases}
  0 & \varepsilon \geq C_1\frac{log^2(k)}{k^{\frac{1}{4}}}\\
d & \varepsilon <
C_1\frac{log^2(k)}{k^{\frac{1}{4}}}
\end{cases}$$
For some constant $C_1$. And thus overall $$log\left(\NN_L(\HH,\varepsilon)\right)\leq C_1^{2}\frac{dlog^4(k)}{k^{1/2}\varepsilon^{2}}= \frac{C_1^{2}log^4(k)}{k^{\frac{1}{4}}\varepsilon^{2}}$$ Which can be as small as we want by choosing large enough $k$, while on the other hand we have that $ADL(\HH)=d=k^{1/4}$.
\end{proof}

\section*{Acknowledgement}
The research described in this paper was funded by the European Research Council (ERC) under the European Union’s Horizon 2022 research and innovation program (grant agreement No. 101041711), the Israel Science Foundation (grant number 2258/19), and the Simons Foundation (as part of the Collaboration on the Mathematical and Scientific Foundations of Deep Learning).

\bibliography{bib}
\end{document}